\RequirePackage{fix-cm}
\documentclass[smallcondensed]{svjour3}     
\smartqed  
\usepackage{graphicx}
\usepackage{amsmath}
\usepackage{amsfonts}
\usepackage{algorithm}
\usepackage[noend]{algpseudocode}
\usepackage{mathptmx}

\usepackage{cite}

\usepackage{dcolumn}
\usepackage{multirow}
\newtheorem{prop}{Proposition}

\usepackage{caption}
\usepackage{booktabs}
\usepackage{url}

\begin{document}

\title{Topological Measurement of Deep Neural Networks Using Persistent Homology
}


\author{Satoru Watanabe         \and
        Hayato Yamana 
}


\institute{S. Watanabe \at
              Graduate School of Fundamental Science and Engineering, Waseda University, Shinjuku-ku, Tokyo, Japan\\
              \email{satoru.watanabe.aw@hitachi.com}           
           \and
           H. Yamana \at
              Graduate School of Fundamental Science and Engineering, Waseda University, Shinjuku-ku, Tokyo, Japan\\
              \email{yamana@waseda.jp}           
}

\date{Received: date / Accepted: date}

\maketitle

\begin{abstract}
The inner representation of deep neural networks (DNNs) is indecipherable, which makes it difficult to tune DNN models, control their training process, and interpret their outputs.
In this paper, we propose a novel approach to investigate the inner representation of DNNs through topological data analysis (TDA).
Persistent homology (PH), one of the outstanding methods in TDA, was employed for investigating the complexities of trained DNNs.
We constructed clique complexes on trained DNNs and calculated the one-dimensional PH of DNNs.
The PH reveals the combinational effects of multiple neurons in DNNs at different resolutions, which is difficult to be captured without using PH.
Evaluations were conducted using fully connected networks (FCNs) and networks combining FCNs and convolutional neural networks (CNNs) trained on the MNIST and CIFAR-10 data sets.
Evaluation results demonstrate that the PH of DNNs reflects both the excess of neurons and problem difficulty, making PH one of the prominent methods for investigating the inner representation of DNNs.

\keywords{Deep neural network \and Convolutional neural network \and Persistent Homology \and Topological data analysis}
\end{abstract}

\begin{acknowledgements}
	We are grateful to Hitachi, Ltd.\  for the tuition subsidy. 
	The founder had no role in study design and technical investigation in this paper. 
\end{acknowledgements}

\section{Introduction}
\label{intro}
Deep neural networks (DNNs) have demonstrated a remarkable performance in various fields including image analysis, speech recognition, and text classification \cite{zhang2018survey,hatcher2018survey}.
However, the inner representations of DNNs are indecipherable, which makes it difficult to tune DNN models, control their training process, and interpret their outputs.
Many approaches enabling the understanding of the inner representation of DNNs have been investigated, including the input identification of specific results \cite{bach2015pixel,zeiler2014visualizing,samek2016evaluating,montavon2017explaining} and similarity evaluation between different networks \cite{raghu2017svcca,morcos2018insights,kornblith2019similarity}.
At the same time, the complexity of DNNs is one of the essential subjects, which represents the knowledge in trained DNNs.  

In this paper, we propose a novel approach to investigate the inner representation of DNNs using topological data analysis (TDA).
TDA employs results from geometry and topology \cite{otter2017roadmap,wasserman2018topological}, which has provided new insights in various fields such as neuroscience \cite{sizemore2018cliques,8113494,curto2017can,yoo2016topological,petri2014homological}, proteomics \cite{cang2018integration,gameiro2015topological,xia2014persistent}, and material science \cite{hiraoka2016hierarchical,kramar2013persistence}.

Persistent homology (PH) is one of the prominent methods in TDA owing to its three advantages: theoretical foundation, computability in practice, and robustness with small perturbations \cite{otter2017roadmap}.
These advantages are beneficial for investigating DNNs.
Theoretical foundation and computability are fundamental in constructing knowledge from empirical observations, while robustness is indispensable for investigating DNNs involving parameter perturbations.

Bastian et al.\ investigated the complexity of the inner representation of DNNs using zero-dimensional PH, which counts the number of connected neurons at different resolutions \cite{rieck2018neural}.
At the same time, one-dimensional PH can reveal other essential aspects of the knowledge complexity in DNNs because it can examine the combinational effects of multiple neurons.
To the best of our knowledge, there is no previous work employing one-dimensional PH for investigating the inner representation of DNNs based on the trained weight parameters except our presentation at a symposium \cite{watanabetopological}. 

We constructed clique complexes, which were employed for analyzing brain networks \cite{reimann2017cliques}, on trained DNNs.
Furthermore, we calculated the one-dimensional PH of fully connected networks (FCNs) and networks combining FCNs and convolutional neural networks (CNNs) trained on the MNIST and CIFAR-10 data set to demonstrate the effectiveness of one-dimensional PH\footnote{The source code used in the evaluation can be accessed at \url{https://github.com/satoru-watanabe-aw/DNNtopology}.}.

The remainder of this paper is organized as follows.
Section \ref{sec:intuition} presents the intuition behind this study.
Background information is presented in Section \ref{sec:background}.
Clique complexes are constructed on trained DNNs in Section \ref{introduceTopology}.
The evaluation setup and results are provided in Section \ref{evaluationSetup} and \ref{result}, respectively. 
Section \ref{discusstion} discusses the assumptions and applications of the measurement method.
Related work is discussed in Section \ref{related}.
Conclusions and suggestions for future work are presented in Section \ref{conclusion}.
 
\section{Intuition behind topological measurement of DNNs}
\label{sec:intuition}
DNNs work as knowledge distilling pipelines, meaning that the degree of feature abstraction increases with the depth of DNN layers \cite{lecun2015deep}.
For example, images of cats are incrementally abstracted from pixels to diagonal lines and ear shapes. Additionally, DNNs can detect cats based on feature combinations \cite{Chollet:2017:DLP:3203489}.
Feature relationships represent the implementation of knowledge in DNNs, which can be investigated from DNN structures.

Previous studies have demonstrated that PH can be used for comparing and characterizing human brains.
Cassidy et al.\ employed PH as a tool for comparing human brains using functional magnetic resonance imaging (fMRI) \cite{8113494}. 
Petri et al.\ demonstrated that psilocybin affects the homological structure of the brain's functional patterns \cite{petri2014homological}. Furthermore,
Sizemore et al.\ employed PH to highlight the crucial features of human brains from diffusion spectrum imaging (DSI) \cite{sizemore2018cliques}.
However, it is often difficult to quantify the activation of neurons from fMRIs and DSIs.
Hence, PH is more useful for analyzing DNNs because their network structures and the activation of neurons can be described mathematically. In this study, we employed PH to investigate the process of training a DNN and evaluate its knowledge representation complexities. 

\section{Background}
\label{sec:background}
The terms of TDA and PH can be understood based on previous studies \cite{edelsbrunner2010computational,horak2009persistent,otter2017roadmap}, while
introductory videos explaining TDA and PH can be found on on-demand video services\footnote{\url{https://www.youtube.com/watch?v=akgU8nRNIp0}, \url{https://www.youtube.com/watch?v=2PSqWBIrn90}}.

\subsection{Persistent homology}
The homology groups of orders zero and one represent the number of connected components and holes, respectively.
PH is a method for computing the homology groups at different resolutions. 
While the formal definition of PH is provided below, 
its intuitive understanding is sufficient for interpreting the presented experimental results obtained using some computational libraries.

\begin{definition} 
	An abstract simplicial complex is a finite collection of sets $\mathcal{K}$ 
	such that $X \in \mathcal{K}$ and $Y\subseteq X$ implies $Y\in \mathcal{K}$.
\end{definition}    

The sets $X$ in $\mathcal{K}$ denote its simplices. 
The dimension of a simplex is $\dim X = {\rm card}\ X - 1$, where ${\rm card}\ X $ denotes the cardinality of $X$.
The dimension of an abstract simplicial complex is the maximum dimension of any of its simplices.
The vertex set is the set consisting of all the simplices of dimension $0$, while the face of a simplex $X$ is a non-empty subset $Y \subseteq X$.

A $p$-chain $c$ of a simplicial complex $\mathcal{K}$ is a formal sum of $p$-simplices in $\mathcal{K}$, that is, $c = \sum a_{i}X_{i}$, where $X_{i}$ are $p$-simplices and $a_{i}$ are coefficients.
We employ module-2 coefficients, that is, $a_{i}$ are either 0 or 1 and $1+1=0$.
The binary arithmetic of two $p$-chains $c = \sum a_{i}X_{i}$ and $c' = \sum b_{i}X_{i}$ is defined as $c + c' = \sum (a_{i} + b_{i}) X_{i}$, where the coefficients are of modulo-2.
The $p$-chain forms a group denoted as $C_{p}$.

A boundary operator $\partial_{p}$ is a map from a $p$-simplex to the sum of its $(p-1)$-simplices. 
Formally, $\partial_{p}X = \sum_{j=0}^{p} [ v_{0},  \ldots,\hat{v_{j}}, \ldots, v_{p}  ]$, where $[ v_{0}, \ldots, v_{p}  ]$ is the simplex with vertices, while the hat indicates that $v_{j}$ is removed.
A chain complex is the sequence of chain groups connected by boundary operators, $\cdots \xrightarrow {\partial_{p+2}}C_{p+1} \xrightarrow{\partial_{p+1}} C_{p} \xrightarrow{\partial_{p}} C_{p-1} \xrightarrow{\partial_{p-1}} \cdots$.
A $p$-cycle is a $p$-chain with an empty boundary forming a group denoted as $Z_{p} = \ker \partial_{p}$.
A $p$-boundary is a $p$-chain, that is, the image of a $(p+1)$-chain forming a group denoted as $B_{p} = {\rm im}\  \partial_{p+1}$.

\begin{definition}
	The $p$-th homology group denoted as $H_{p} (= Z_{p}/B_{p})$ is the $p$-th cycle group modulo the $p$-th boundary group.
	The $p$-th Betti number $\beta_{p}$ is the rank of $H_{p}$.
\end{definition} 

\begin{definition}
	A filtration of the simplicial complex $\mathcal{K}$ is a sequence of simplicial complex such that $\emptyset = K_{0}\subset K_{1} \subset \cdots\subset K_{n} = \mathcal{K}$.
\end{definition} 

For every $i \le j$, there is an induced homomorphism in each dimension $p$, $f_{p}^{i,j} $ from $H_{p}(K_{i})$ to $H_{p}(K_{j})$.
$f_{p}^{i,j}$ satisfies the condition of $f_{p}^{k,j}\circ f_{p}^{i,k} = f_{p}^{i,j}$ for all $0 \le i \le k \le j \le n$.

\begin{definition}
	Let $\emptyset = K_{0}\subset K_{1} \subset \cdots\subset K_{n} = \mathcal{K}$ be a filtration.
	The $p$-th PH of $\mathcal{K}$ is the pair $( \{H_{p}(K_{i})\}_{0\le i \le n},\{f_{p}^{i,j}\}_{0\le i \le j \le n})$, 
	where the homomorphism $f_{p}^{i,j} : H_{p}(K_{i} ) \rightarrow H_{p}(K_{j})$ represents the maps induced by including maps $K_{i} \rightarrow K_{j}$.
\end{definition} 

A homology $\gamma \in H_{p}(K_{i}) $ can be said to be born at $K_{i}$ if $\gamma \notin {\rm im} f_{p}^{i-1, i}$.
Furthermore, if $\gamma$ is born at $K_{i}$, then it dies entering $K_{j}$ if $f_{p}^{i,j-1}(\gamma) \notin {\rm im} f_{p}^{i-1,j-1}$ but $f_{p}^{i,j}(\gamma) \in {\rm im}f_{p}^{i-1,j}$.
The lifetime of $\gamma$ is represented by the half-open interval $[i,j)$. 
If $f_{p}^{i,j}(\gamma) \neq 0 \ (i \le \forall j \le n)$, $\gamma$ can be said to live forever, and its lifetime is the interval $[i,\infty)$.

\begin{figure*}[t]
	\centering
	\includegraphics[width=\hsize]{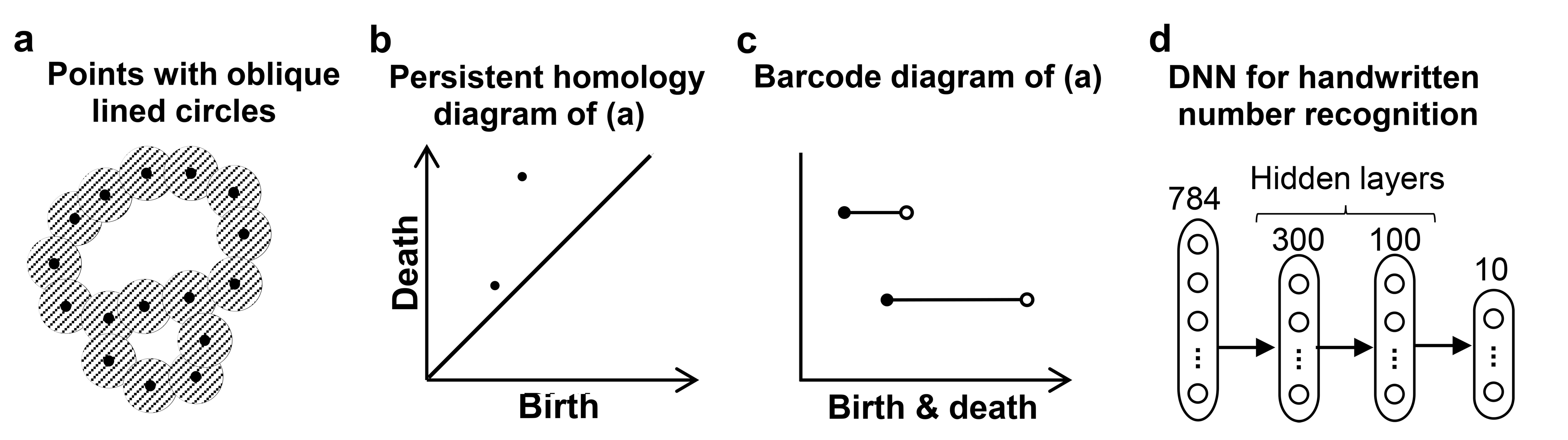}
	\caption{(a) Examples of persistent homology diagrams; (b) persistent homology diagram of (a); (c) barcode diagram of (a); (d) DNN for handwritten number recognition.}
	\label{fig:figure1}
\end{figure*}

\subsection{Diagrams}
A PH diagram illustrates the birth and death of homologies in a filtration, which was fundamentally introduced in \cite{barannikov1994framed}.
Fig. \ref{fig:figure1}(a) shows points with oblique 
lined circles in $\mathbb{R}^{2}$.
When the radius of the circles is small, the points are isolated.
Two encircled regions appear in $\mathbb{R}^{2}$ when the circles are gradually enlarged.
The appearance of the encircled regions corresponds to the birth of homologies.
The regions disappear when the circles are enlarged further, and the disappearances correspond to the death of homologies. 

Fig. \ref{fig:figure1}(b) shows the PH diagram of Fig. \ref{fig:figure1}(a), in which the X-axis shows the birth of homologies and the Y-axis the death of them.  
The two points in Fig. \ref{fig:figure1}(b) correspond to the births and deaths of the two regions. 
The large region in Fig. \ref{fig:figure1}(a) is stable with regard to the enlargement of the circles.
In contrast, the small region is less stable compared to the large region.
The stability of the regions is indicated by the distance from the dialog line in Fig. \ref{fig:figure1}(b), i.e., the small region is pointed near the dialog line, whereas the large region is pointed in a distance from the dialog line.

Barcode is another diagram that gives the same information as the PH diagram. 
Barcode diagram of Fig. \ref{fig:figure1}(a) is shown in Fig. \ref{fig:figure1}(c), in which the start and end points of lines parallel to the X-axis show the birth and death of homologies, respectively.
The short and long lines correspond to the small and large regions, respectively. 
The stability of regions is indicated by the length of the bars in the barcode diagrams.

\section{Construction of clique complexes on DNNs}\label{introduceTopology}
We consider a set of neurons as vertices $V = \{v_{0}, \ldots, v_{n} \}$, where $n+1$ is the number of neurons. 
DNNs are considered as directed graphs with weights $w_{ij}$, where $w_{ij}$ denotes the weight between $v_{i}$ and $v_{j}$; here, $w_{ij}$ is zero if $v_{i}$ and $v_{j}$ are not connected.
We set the value of the relevance of identical neurons to one and the relevance $R_{ij}$ between the connected neurons $v_{i}$ and $v_{j}$ as the normalized weight.
Formally we set 
\begin{eqnarray}\label{L}
R_{ij} = \left \{ \begin{array}{ll}
1 & ( i = j ) \\
w_{ij}^{+}/\sum_{i,i\neq j} w_{ij}^{+} &(i \neq j),
\end{array}  \right.
\end{eqnarray}
where $w_{ij}^{+}$ denotes the positive part of the weight, i.e.\ $w_{ij}^{+} = \max \{ 0, w_{ij} \}$.
$R_{ij}$ indicates the relevance between $v_{i}$ and $v_{j}$ because the input to the j-th neuron is calculated by $\sum_{i} a_{i}w_{ij} + b_{j}$ in DNNs, where $a_{i}$ is the activation of the i-th neuron and $b_{j}$ is the bias \cite{Chollet:2017:DLP:3203489}.
We employed the positive part of the weight and ignored the bias, in a manner similar to the $z^{+}$-rule defined in deep Taylor decomposition \cite{montavon2017explaining}.

To construct clique complexes on DNNs, 
the relevance was extended to indirectly connected neurons.
For example, when $v_{0}$ and $v_{2}$ are connected to a path $v_{0} \rightarrow v_{1} \rightarrow v_{2}$, the relevance between $v_{0}$ and $v_{2}$ corresponding to the path is defined as $R_{01}R_{12}$.
The intuition behind the definition is as follows:
$R_{01}$ and $R_{12}$ indicate the contributions of $v_{0}$ and $v_{1}$ to the increase in the inputs of $v_{1}$ and $v_{2}$, respectively; $R_{01}R_{12}$ indicates the contribution of $v_{0}$ to the increase in the input of $v_{2}$. 
Formally we set 
\begin{equation}\label{R}
\widetilde{R_{ij}} = \max_{ {(v_{i}, v_{m_{1}}\ldots,v_{m_{k}}, v_{j})}\in L_{ij}} R_{v_{i}v_{m_{1}}}\cdots R_{v_{m_{k}}v_{j}},
\end{equation}
where $L_{ij}$ denotes the set of all possible paths from $v_{i}$ to $v_{j}$.
It is possible to define $\widetilde{R_{ij}}$ using multiple paths in $L_{ij}$. However, the maximum was employed in Eq. (\ref{R}) to  improve computational efficiency.

Masulli et al.\ constructed a clique complex $K(G)$ on a finite directed weighted graph $G = (V, E)$ with vertex set $V$ and edge set $E$ with no self-loops and no double edges \cite{masulli2016topology}.
They defined the clique complex $K(G)$ as $K(G)_{0} = V $ and $K(G)_{p} = \{ (v_{K_{0}},\ldots, v_{K_{p}}) \ ;\ v_{K_{i}}\in V, (v_{K_{i}},v_{K_{j}}) \in E \mbox{ for all\ } K_{i} < K_{j} \} \ (\mbox{for \ } p \geq 1)$, where $K(G)_{p}$ denotes the set of $p$-simplices on $G$.

Correspondingly, $\widetilde{R_{ij}}$ enables the construction of a clique complex and filtration on $V$.
The neurons were numbered in ascending order from the output to input layers. 
Hence, the numbers of neurons in the closer layer to the output layer are smaller than those in the farther layer, where the distance is indicated by the number of edges from the output layer. 
Using this numbering, we set $p$-simpleces on $V$ as 
\begin{eqnarray}\label{simpleces}
K_{p}^{t} = \left \{ \begin{array}{ll}
V & ( p = 0 ) \\
\{ (v_{k_{0}},\ldots, v_{k_{p}})\ ;\ v_{k_{i}}\in V, \widetilde{R_{k_{i}k_{j}}} \geq t\ \mbox{ for all\ } k_{i} > k_{j} \}  &(p \geq 1),
\end{array}  \right.
\end{eqnarray}
where $t$ is a threshold value ($0\leq t\leq 1$).

\begin{prop}
	Let $V = (v_{0}, \ldots,v_{n})$ be a finite set, and $\{w_{ij}\} \ (0\leq i,j \leq n)$ be a set of real numbers.
	Let $\widetilde{R_{ij}} \ (0\leq i, j \leq n)$ be the relevance defined by Eqs. (\ref{L}) and (\ref{R}) using $\{w_{ij}\}$. Let
	$K_{p}^{t}$ be the $p$-simplices defined by Eq. (\ref{simpleces}), where $t$ is a threshold value ($0\leq t \leq 1$).
	Then, a finite collection of sets $K^{t} = K_{0}^{t}\cup K_{1}^{t}\cup \cdots \cup K_{n}^{t}$ is an abstract simplicial complex.
\end{prop}
\begin{proof}
	Let $X = \{ v_{X_{0}}, \ldots, v_{X_{p}}\}$ be an element of $K^{t}$. Then, $\widetilde{ R_{X_{i}X_{j}}}$ is greater than or equal to $t\ \mbox{for all}\ X_{i} > X_{j}$.
	Let $Y = \{ v_{Y_{0}},\ldots, v_{Y_{q}}\}$ be a subset of $X$.
	Then, $\widetilde{R_{Y_{i}Y_{j}}} $ are greater than or equal to $t\ \mbox{for all}\ Y_{i} > Y_{j}$.
	Therefore, $X \in K^{t}$ and $Y\subseteq X$ imply $Y\in K^{t}$. 
	\qed
\end{proof}

\begin{prop}
	Let $(t_{i})_{i=1}^{n}$ be a monotonically decreasing sequence ranging from $1$ to $0$.
	Then, $K_{0} = \emptyset $ and $K_{i} = K^{t_{i}}\ (1 \leq i \leq n)$ form a filtration of $K^{t_{n}}$. 
\end{prop}
\begin{proof}
	$K_{p}^{t_{k}}$ is included in $K_{p}^{t_{l}}$ $(1 \geq t_{k} > t_{l} \geq 0)$ from Eq. (\ref{simpleces}). 
	It implies $\emptyset = K_{0} \subset K_{1} \subset \cdots \subset K_{n} = K^{t_{n}}$.
	\qed
\end{proof}

\begin{figure*}[t]
	\centering
	\includegraphics[width=\linewidth]{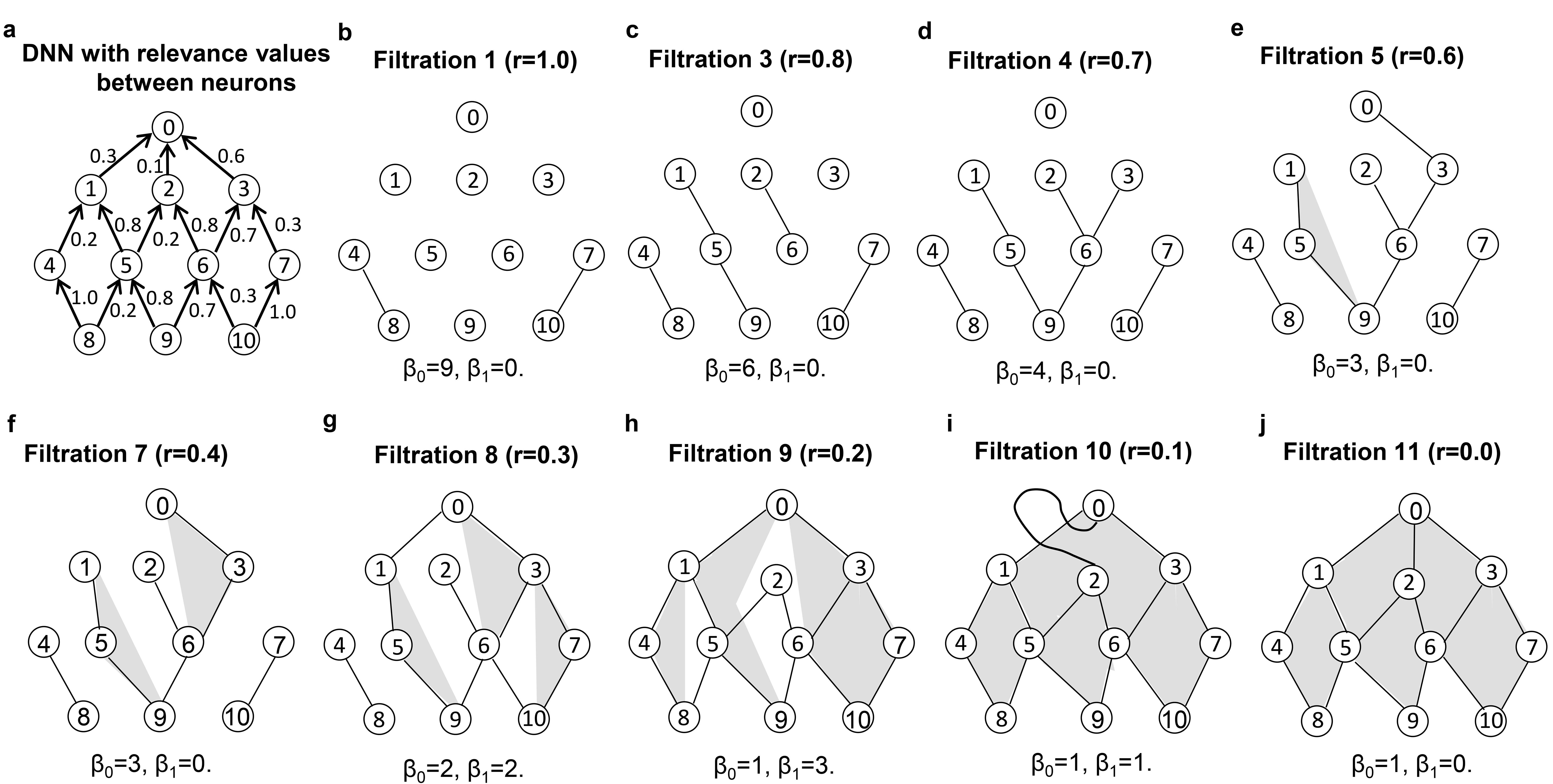}
	\caption{(a) Example of DNN with weights; (b--h) simplicial complexes and betti numbers corresponding to the filtration.}
	\label{fig:figure2}
\end{figure*}

\begin{figure*}[t]\label{relevance}
	\centering	\includegraphics[width=\linewidth]{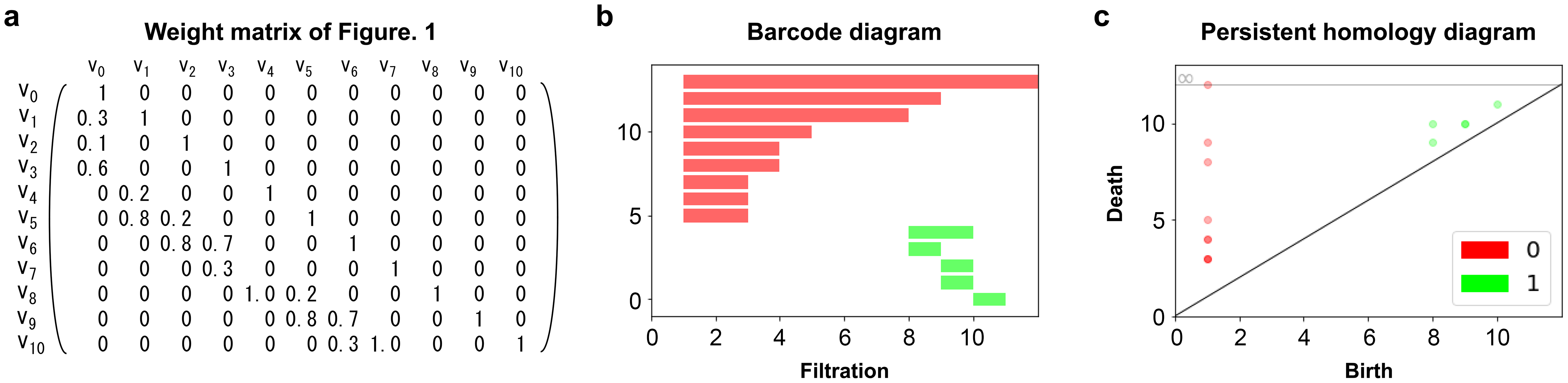}
	\caption{(a) Weight matrix of Fig. 2(a); (b,c) barcode and PH diagrams illustrated using GUDHI library.}
	\label{fig:figure3}
\end{figure*}

Fig. \ref{fig:figure2}(a) illustrates a four-layered DNN with an output neuron $v_{0}$.
The values adjacent to the arrows denote the weight between two neurons, and the weight matrix is presented in Fig. \ref{fig:figure3}(a) where the (i,j) element denotes the weight between the i-th and j-th neurons.
Fig. \ref{fig:figure2}(b) illustrates the simplicial complex of $K_{r=1.0}$ with Betti number $\beta_{0} = 9$.
The decrease of the Betti number $\beta_{0}$ according to the filtration can be observed in Fig. \ref{fig:figure2}(c) to (h). 
Fig. \ref{fig:figure2}(e) illustrates a 2-simplex represented with the gray triangle.

Fig. \ref{fig:figure2}(g) and \ref{fig:figure2}(h) illustrate the increase of the Betti number $\beta_{1}$ corresponding to the occurrences of the cycle.
If the vertices representing the features of input images are connected straightforwardly to the output neurons, the knowledge in the DNN is considered to be simple because it is equivalent to feature detection. 
In contrast, the increase of the Betti number $\beta_{1}$ indicates that the DNN classifies the input based on the combination of features.
From these viewpoints, we can assume the increase in the Betti number $\beta_{1}$ reflects the complexity of knowledge in the DNN.
Filtration 10 (Fig. \ref{fig:figure2}(i)) has Betti number $\beta_{1} = 1$.
While [0, 2] is a simplex in Filtration 10, it is not included in another simplex [0,$\ldots$,10] and produces $\beta_{1} = 1$.

The computation of PH involves the explosion of the complexity caused by the increase of vertices, several implementations of which are publicly available \cite{otter2017roadmap}. 
We employed the GUDHI \cite{gudhi:cython,gudhi:urm,boissonnat2014simplex}, JavaPlex \cite{Javaplex}, and Dionysus 2 \cite{edelsbrunner2012persistent,edelsbrunner2000topological} libraries for the computation and visualization.  
These libraries require registering simplexes in each filtration to calculate PH.

Algorithm \ref{algorithm} identifies all simplexes from a vertex $s$ up to the limit of the threshold of relevance $t$ using the recursive procedure call. 
All simplexes in each filtration are identified using this procedure and registered to the libraries.
Fig. \ref{fig:figure3}(b) and (c) are barcode and PH diagrams illustrated by the GUDHI library, respectively.
The library employed red and green for indicating zero- and one-dimensional homologies, respectively.
The Betti numbers in Fig. \ref{fig:figure3}(b) correspond to the number of the intersections between the bars and the perpendicular lines to the X-axis (remembering that the lifetime of homologies is defined by the half-open interval $[birth, death)$).
The GUDHI library illustrates Betti numbers using color shades in PH diagrams shown in Fig. \ref{fig:figure3}(c).
PH was calculated using the Dionysus 2 and JavaPlex libraries, resulting in the same diagrams. 

\begin{algorithm*}[t]
	\captionof{algorithm}{Algorithm for obtaining simplexes from a vertex $s$ using a threshold $t$}
	\label{algorithm}
	\begin{algorithmic}
		\Procedure{getSimplex}{$M$, $s$, $t$} \Comment{ where $M$: $n\times n$-matrix, $s$: array, $t$: threshold}
		\State{ $relevance \gets 1.0$, $result \gets \emptyset$, $origin \gets s[0]$}
		\For {$dest = s[0]$ to $s[|s|-1]$} \Comment{calculate the relevance from $s[0]$ to $s[|s|-1]$.}
		\State{$relevance \gets relevance \times M[origin][dest]$}\Comment{$s[|s|-1]$ is the last element of s.}
		\State{$origin \gets dest$}
		\EndFor
		\If {$relevance \geq t$}
		\State{$result.append(combination(s))$}\Comment{append all the combinations of the elements in s.}	
		\State{$lastPoint\gets s[|s|-1]$} 
		\For{ $i = 0$ to $n-1$}\Comment{check if the last point has connections.}
		\If{ $M[lastPoint][i] > 0$ and $i \neq lastPoint$}
		\State{$ss \gets $deep copy of $s$}
		\State{$recResult\gets getSimplex(M, ss.append(i), t)$}
		\Comment{recursive call with extended array.}
		\For{ $e$ in $recResult$}
		\State{$result.append(combination(e))$} \Comment{append all the combinations of the elements in e.}
		\EndFor
		\EndIf
		\EndFor
		\EndIf
		\Return{$unique(result)$}
		\Comment{return deduplicated array}
		\EndProcedure
	\end{algorithmic}
\end{algorithm*}

A filtration is defined using thresholds of relevance.
This study considered 64 threshold values composed with $(1.0^{0},\ldots ,1.0^{-7})$ and eight interval values between the adjacent values.
Formally, we considered the simplicial complexes $K_{n(r=(1-0.1\times(l-1))\times 10^{-m})} (1\le n\le 64)$, where $m$ and $l$ are the quotient and remainder when $n$ is divided by $9$, respectively.
And the filtration was defined as  $K_{1 (r=1.0)}\subset K_{2 (r=0.9)}\subset \cdots \subset K_{10 (r=1.0^{-1})}\subset  K_{11 (r=0.09)}\subset\cdots\subset K_{64 (r=1.0^{-7})}$.
While the thresholds should be considered depending on the network structure of DNNs, we set this aside as a task for future work; this study only examined the prominence of the topological measurement of DNNs.

\section{Evaluation setup}\label{evaluationSetup}
The MNIST and CIFAR-10 data sets were employed in the evaluation \cite{lecun1998gradient,krizhevsky2009learning}.
As shown in Table \ref{dataset}, the contents of the MNIST and CIFAR-10 data sets are $28\times 28$ grayscale handwritten digits and $32 \times 32$ color photographs, respectively. 
The CIFAR-10 data set comprises the photographs of 10 types of objects such as airplanes, automobiles, birds, etc.
All experiments were conducted using Keras and Tensorflow \cite{Chollet:2017:DLP:3203489,abadi2016tensorflow}, and DNNs were developed based on the examples in Keras 2.3.0.

For the classification of the MNIST data set, we employed an FCN with two hidden layers of sizes 300 and 100, the ReLU activation function in the hidden layers and 10 output neurons with the sigmoid activation function (Fig. \ref{fig:figure1}(d)).
The models were traind for 10 epochs with a batch size of 64, and all models achieved an accuracy of over 97\% on the test data.

\begin{table}[t]
	\caption{Overview of the data sets and network types employed in this study}
	\label{dataset}
	\begin{center}
		\begin{tabular}{lllll}
			\hline\noalign{\smallskip}
			Data set        & Content & Data size & Network type \\
			\noalign{\smallskip}\hline\noalign{\smallskip}
			MNIST           & handwritten digits &784 ($28\times 28$ grayscale) &  FCN\\
			CIFER-10        & photographs & 3072 ($32\times 32$ color) & CNN, FCN \\
			\noalign{\smallskip}\hline
		\end{tabular}
	\end{center}
\end{table}

For the classification of the CIFAR-10 data set, we employed DNNs consisting of a CNN and an FCN.
The CNN was used to extract features from the photographs, while the FCN was used to classify the photographs based on the combination of the features.
The proposed method was applied to the FCN since the purpose of this study was to examine the complexity of the knowledge in DNNs represented in the combination of features. 

We employed the CNN from an example network included in Keras 2.3.0 without modifications.
This CNN comprises multiple layers, including two-dimensional convolution, max pooling, and dropout layers.
Two FCNs with sizes of (300, 100, 10) and (512, 512, 10) were used for examining the sensitivity of the proposed method to the network structures\footnote{The following network structures are employed: input(3072)--Conv2D(32 filters, $3\times3$ kernel, ReLu activation)--Conv2D(32 filters, $3\times3$ kernel, ReLu activation)--MaxPooling2D($2\times$2 pool)--Dropout(dropout  ratio 0.25)--Conv2D(64 filters, $3\times3$ kernel, ReLu activation)--Conv2D(64 filters, $3\times3$ kernel, ReLu activation)--MaxPooling2D($2\times2$ pool)--Dropout(dropout ratio 0.25)--Flatten--Dense(300 or 512, ReLu activation)--Dropout(dropout ratio 0.5)--Dense(100 or 512, ReLu activation)--Dense(10, softmax activation).}.
The DNNs were trained for 30 epochs with a batch size of 32. 

\section{Evaluation results}\label{result}
\subsection{MNIST data set}\label{MNISTResult}
Figs. \ref{fig:MNIST300100}(a--j) illustrate PH diagrams of the FCNs produced using the Dionysus 2 library, where the number of input digits used to train the FCN models was varied. 
In particular, we extracted the images of the target digits from the MNIST data set and trained FCN models using the images of digits 0--9 (Fig. \ref{fig:MNIST300100}(a)), digits 0--8 (Fig. \ref{fig:MNIST300100}(b)), and so on.
The Dionysus 2 library allows to visualize the overlapping quantity of homologies using different colors as indicated by the legends in Fig. \ref{fig:MNIST300100}.
The values of birth and death in the axes on PH diagrams indicate the order of the 64 threshold values defined in Section \ref{introduceTopology}.
Let $m$ and $l$ are the quotient and remainder when the values of birth and death are divided by $9$, respectively, the threshold values corresponding to the values in the axes on PH diagrams are $(1-0.1\times(l-1))\times 10^{-m}$.
This correspondence is consistent through the paper.

The following three observations can be made from Figs. \ref{fig:MNIST300100}(a--j): (1) points are plotted in the belt-like area ($birth + 5 < death < birth + 20$) parallel to the dialog line; (2) some figures have points below the belt-like area; and (3) some figures have points over the belt-like area.

With respect to observation (2), the number of points below the belt-like area increases from Fig. \ref{fig:MNIST300100}(a) to Fig. \ref{fig:MNIST300100}(g) and decreases from Fig. \ref{fig:MNIST300100}(h) to Fig. \ref{fig:MNIST300100}(j).
This pattern reflects both the excess of the output neurons and problem difficulty. It can be further observed that the diagrams seem to reflect the degree of confidence of the FCN models, i.e., the excess of the output neurons reduced the confidence, whereas the simplicity of the problem increases it.
For further investigation, we classified five digits using five output neurons (Fig. \ref{fig:MNIST300100}(k)) and 10 digits using 20 output neurons (Fig. \ref{fig:MNIST300100}(l)).
In contrast to Fig. \ref{fig:MNIST300100}(f), the points below the belt-like area disappeared in Fig. \ref{fig:MNIST300100}(k). 
The opposite can be observed in Figs. \ref{fig:MNIST300100}(a) and \ref{fig:MNIST300100}(l).

Table \ref{NumberofPoint} lists the number of points plotted in Fig. \ref{fig:MNIST300100}(a--e), \ref{fig:MNIST300100}(i), and \ref{fig:MNIST300100}(j).
We categorized the points using the representative cycles calculated by the JavaPlex based on the following two conditions: (c1) the homology includes unused output neurons and (c2) the points are under the belt-like area ($death \le birth + 5$).
While the number of points that include unused output neurons in Figs. \ref{fig:MNIST300100}(i) and \ref{fig:MNIST300100}(j) is more than twice of that in Fig. \ref{fig:MNIST300100}(e), these points are not plotted below the belt-like area.
The simplicity of the problem led to no points being plotted under the belt-like area.

 \begin{figure*}[t]
	\centering
	\includegraphics[width=\linewidth]{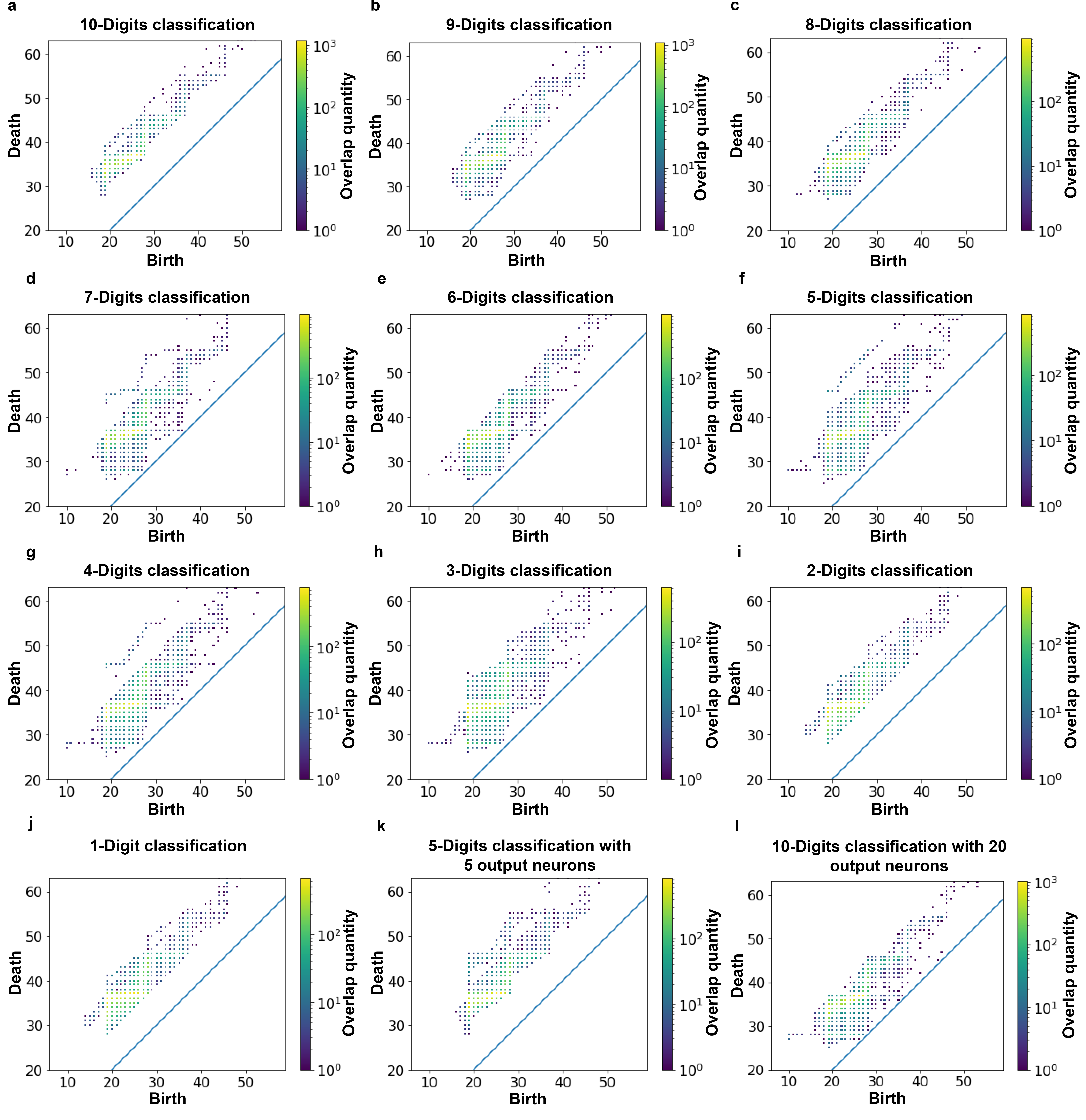}
	\caption{(a--j) PH
		 diagrams of the FNC models trained to classify handwritten digits based on a varying number of input digits from 10 to 1; (k) persistent diagram of the FCN model trained to classify five digits using five output neurons; (l) persistent diagram of the FCN model trained to classify 10 digits using 20 output neurons}
	\label{fig:MNIST300100}
\end{figure*}

\begin{table*}
	\caption{Number of points in Figs. \ref{fig:MNIST300100}(a--e), \ref{fig:MNIST300100}(i), and \ref{fig:MNIST300100}(j)}
	\label{NumberofPoint}
	\centering
	\begin{tabular}{lccccccc}
		\toprule
		& (a) & (b) & (c)&(d)&(e) &(i)& (j) \\
		\midrule
		Total number & 16,420  & 16,399 & 16,150 & 16,222 & 16,133 &15,857 &15,531 \\
		(c1) & N/A & 1,317 & 2,034 & 1,700 &2,972 & 8,226&13,123   \\
		(c2) & 0 & 45 & 26 & 254 & 273 &0& 0     \\
		(c1) and (c2)
		& N/A       & 45 & 26 & 254 &40& 0&0  \\
		\bottomrule
	\end{tabular}
\end{table*}

\subsection{CIFAR-10 data set}\label{CIFAR10Result}
Figs.\ref{fig:CIFAR300100}(a--j) illustrate PH diagrams of the DNN models combining a CNN and an FCN (300, 100, 10), where the number of classes used to train the models was varied. 
In particular, we extracted photographs of the target classes from the CIFAR-10 data set and trained the DNN models using the photographs of 10 classes (Fig. \ref{fig:CIFAR300100}(a)), nine classes (Fig.\ref{fig:CIFAR300100}(b)), and so on.

As described in Section \ref{evaluationSetup}, the contents of the CIFAR-10 data set differs from that of the MNIST data set in terms of the image size, tone, and represented object.
Unlike FCN-based models traind on the MNIST data set, CNNs were employed in addition to FCNs to classify the CIFAR-10 data set.
 
Despite these differences, Figs. \ref{fig:CIFAR300100} demonstrate similar patterns to those in 
Figs. \ref{fig:MNIST300100}.
In particular, the points under the belt-like area appear only in Figs. \ref{fig:CIFAR300100}(d--h); 
Fig. \ref{fig:CIFAR300100}(k), where the photographs of five classes are classified using five output neurons, has no points under the belt-like area, whereas Fig. \ref{fig:CIFAR300100}(l), where the photographs of 10 classes are classified using 20 output neurons, has points under the belt-like area.

A further experiment was conducted using the DNN models combining a CNN and an FCN (512, 512, 10). 
The results of this experiment are illustrated in Figs. \ref{fig:CIFAR512}.
A similar patterns regarding the appearance and disappearance of points under the belt-like area can be observed from Fig. \ref{fig:CIFAR512}; that is, only Figs. \ref{fig:CIFAR512}(d--h) and \ref{fig:CIFAR512}(l) have the points under the belt-like area.
This result suggests that the observation is robust to not only the network type and content of data sets but also number of neurons in FCNs.

Two additional observations can be made from Fig. \ref{fig:CIFAR300100} and \ref{fig:CIFAR512}: (i) the numbers of points in Figs. \ref{fig:CIFAR512} are larger than those in Figs. \ref{fig:CIFAR300100}; (ii) the sizes of the areas that points are plotted in Figs. \ref{fig:CIFAR512} are larger than those in Figs. \ref{fig:CIFAR300100}.
Tables \ref{pointCompare} and \ref{convexHull} list the numbers of points and sizes of the convex hull of the points plotted in Fig. \ref{fig:CIFAR300100}(a)--(j) and \ref{fig:CIFAR512}(a)--(j), respectively.
The numbers of points in Fig. \ref{fig:CIFAR512} are 8.81 to 9.31 times larger than those in Fig. \ref{fig:CIFAR300100}.
The sizes of the convex hulls in Fig. \ref{fig:CIFAR512} are 1.05 to 2.57 times larger than those in Fig. \ref{fig:CIFAR300100}.

\begin{figure*}
	\centering
	\includegraphics[width=\linewidth]{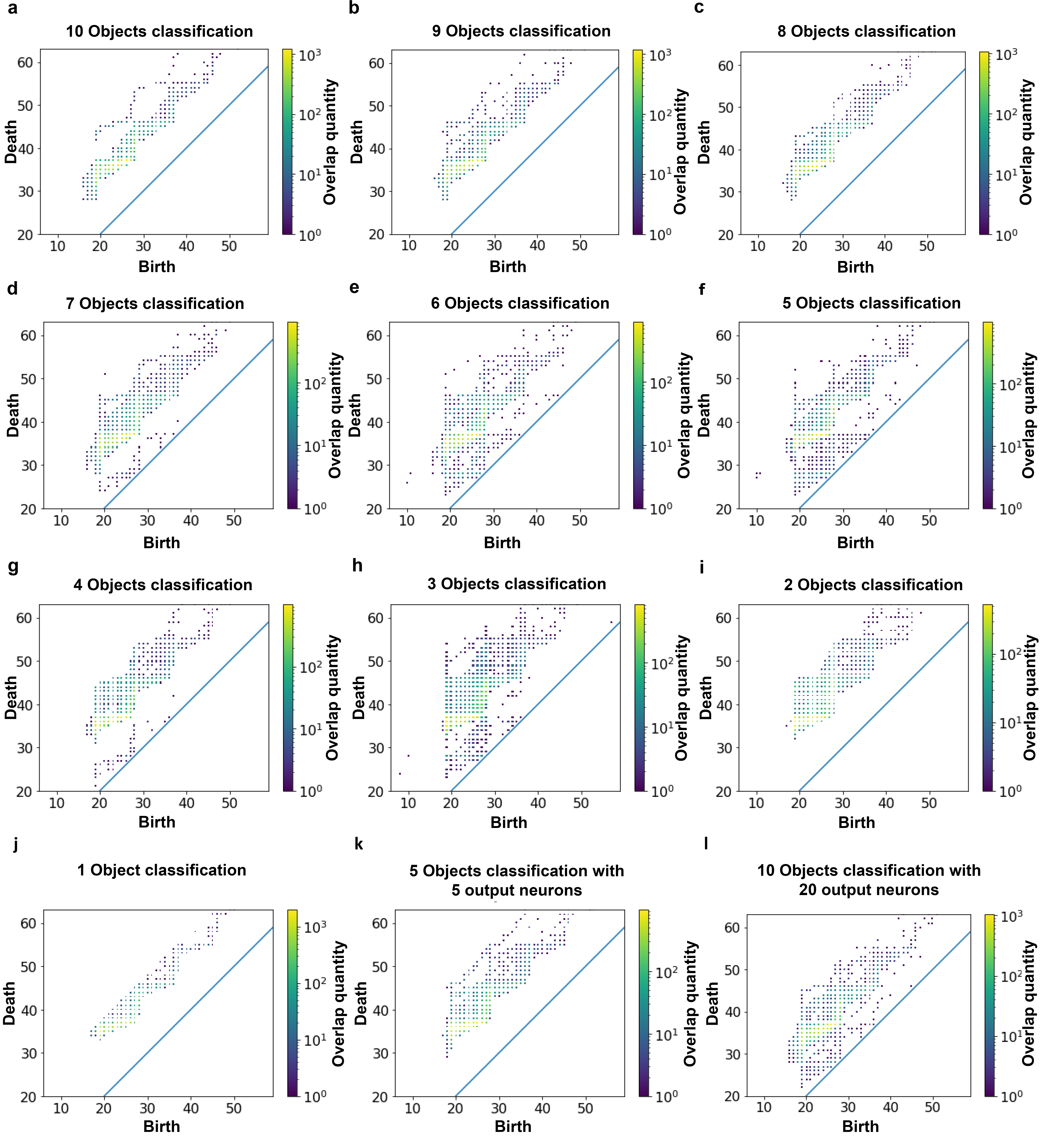}
	\caption{(a--j) PH diagrams of the DNNs using the FCN (300, 100, 10) trained to classify photographs based on a varying number of input classes from 10 to 1; (k) PH diagram of the DNN using the FCN (300, 100, 10) trained to classify five classes using five output neurons; (l) PH diagram of the DNN using the FCN (300, 100, 10) trained to classify 10 classes using 20 output neurons}
	\label{fig:CIFAR300100}
\end{figure*}

The number of points reflects the difference of expressiveness of the FCN (512, 512, 10) and FCN (300, 100, 10).
The FCN (512, 512, 10) has more parameters compared to the FCN (300, 100, 10), which results in 
the ability of the FCN (512, 512, 10) to learn knowledge is higher than that of the the FCN (300, 100, 10) and produces many homologies.
As a rough approximation, the FCN (512,512,10) has $512\times 512 + 512 \times 10$ of weight parameters, whereas the FCN (300, 100, 10) has $300\times 100 + 100 \times 10$ of them. 
The ratio 8.62 ($= (512\times 512 + 512 \times 10) / (300\times 100 + 100 \times 10)$) provides the explanation for the increase in the values listed in Table \ref{pointCompare}. 

\begin{figure*}
	\centering
	\includegraphics[width=\linewidth]{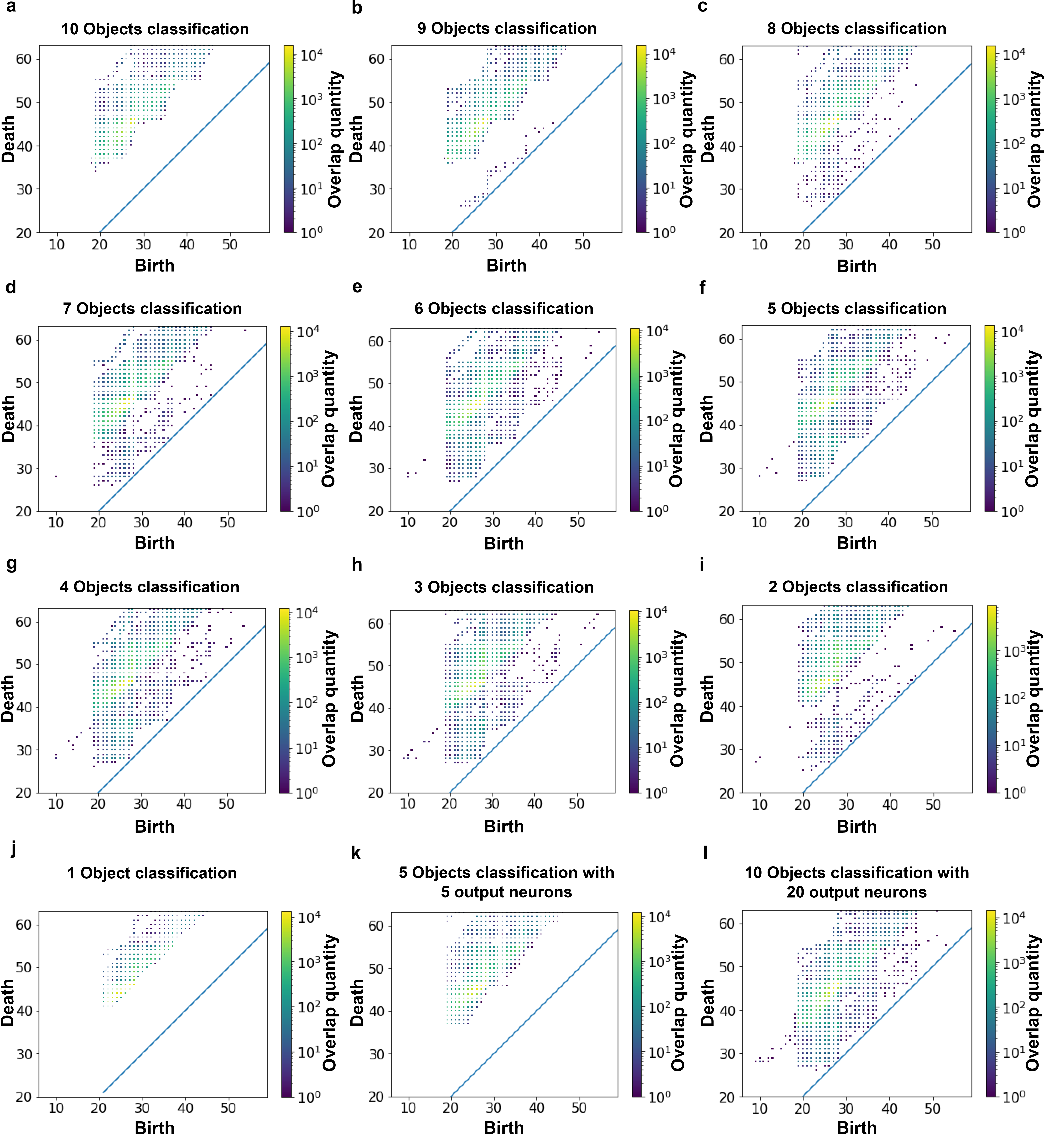}
	\caption{(a--j) PH diagrams of the DNN using the FCN (512, 512, 10) trained to classify photographs based on a varying number of input classes from 10 to 1; (k) PH diagram of DNN using the FCN (512, 512, 10) trained to classify five classes using five output neurons; (l) PH diagram of DNN using the FCN (512, 512, 10) trained to classify 10 classes using 20 output neurons}
	\label{fig:CIFAR512}
\end{figure*}

The increase in the size of convex hull is smaller than that of the number of points, which indicates that the FCNs (512, 512, 10) have duplicated homologies approximately 4 to 8 times more often compared to the FCNs (300, 100, 10).
It implies that the FCNs (512, 512, 10) have duplicated homologies with different neurons, which can be achieved  with expressive training to the data set.
The interpretation of the PH diagrams requires further investigation, which we left as a task for future work because the purpose of this study was only to examine the prominence of the topological measurement of DNNs. 

 \begin{table*}
	\caption{Number of points in Figs.\ref{fig:CIFAR300100}(a--j) and \ref{fig:CIFAR512}(a--j)}
	\label{pointCompare}
	\centering
	\renewcommand{~}{\phantom{0}}
	\begin{tabular}{cccc}
		\toprule
	    & \begin{tabular}{c}
			(A) Fig. \ref{fig:CIFAR300100}: FCN\\(300, 100, 10)
		\end{tabular} 
		&  \begin{tabular}{c}
			(B) Fig. \ref{fig:CIFAR512}: FCN\\(512, 512, 10)
		\end{tabular} 
		& \begin{tabular}{c}
			(B) / (A)
		  \end{tabular}
		\\
		\midrule
		(a) & 16,214 & 142,768 & 8.81 \\
		(b) & 16,278 & 139,783 & 8.59 \\
		(c) & 15,702 & 142,016 & 9.04 \\
		(d) & 15,421 & 141,027 & 9.15 \\
		(e) & 15,274 & 138,732 & 9.08 \\
		(f) & 15,759 & 136,508 & 8.66 \\
		(g) & 14,878 & 133,503 & 8.97 \\
		(h) & 14,348 & 124,919 & 8.71 \\
		(i) & 11,496 & 106,983 & 9.31 \\
		(j) & 15,073 & 132,775 & 8.81 \\
		\bottomrule
	\end{tabular}
\end{table*}

 \begin{table*}
	\caption{Size of the convex hull in Figs. \ref{fig:CIFAR300100}(a--j) and \ref{fig:CIFAR512}(a--j)}
	\label{convexHull}
	\centering
	\renewcommand{~}{\phantom{0}}
	\begin{tabular}{cccc}
		\toprule
		& \begin{tabular}{c}
			(A) Fig. \ref{fig:CIFAR300100}: FCN\\(300, 100, 10)
		\end{tabular} 
		&  \begin{tabular}{c}
			(B) Fig. \ref{fig:CIFAR512}: FCN\\(512, 512, 10)
		\end{tabular} 
		& \begin{tabular}{c}
			(B) / (A)
		\end{tabular}
		\\
		\midrule
		(a) & 445.5 & ~492.5 & 1.11 \\
		(b) & 477.0 & ~737.5 & 1.55 \\
		(c) & 406.0 & ~881.0 & 2.17 \\
		(d) & 710.5 & 1029.5 & 1.45 \\
		(e) & 823.0 & ~959.5 & 1.17 \\
		(f) & 836.0 & ~904.8 & 1.08 \\
		(g) & 634.5 & ~964.5 & 1.52 \\
		(h) & 992.0 & 1041.5 & 1.05 \\
		(i) & 413.5 & 1061.0 & 2.57 \\
		(j) & 232.5 & ~254.0 & 1.09 \\
		\bottomrule
	\end{tabular}
\end{table*}

\subsection{Robustness on weight initialization}
We conducted additional experiments by varying the initial values of network weights to investigate the robustness of the PH diagrams' transitions described in Subsections \ref{MNISTResult} and \ref{CIFAR10Result}.
Keras framework starts the training with random initial values of network weights \cite{Chollet:2017:DLP:3203489}.
We repeated each experiment 10 times by varying the number of input classes from 10 to 1 with the three network types, MNIST $(300, 100, 100)$, CIFAR-10 $(300, 100, 100)$, and CIFAR-10 $(512, 512, 10)$, resulting in a total of 300 additional experiments.

Fig. \ref{fig:sizeofconvex} shows the minimum, average, and maximum size of convex hulls of the points in the PH diagrams.
The differences between the maximum and minimum values indicate the degree of vibration of the experiment results.
All the three graphs are approximately convex upward, indicating that the PH diagrams transit the shape in a similar manner to those described in Subsections \ref{MNISTResult} and \ref{CIFAR10Result}, and the transitions are robust on the initial values of network weights.

In Subsections \ref{MNISTResult} and \ref{CIFAR10Result}, we observed the transition of the PH diagrams that the number of points near the dialog line ($death \le birth +5$) changes by varying the number of input classes.
No point near the dialog line appeared when the number of input classes was set to 10 and 1. 
Additionally, the number of points near the dialog line increased and decreased with the decrease in the number of input classes from 10 to 8 and 3 to 1, respectively.

Table \ref{NumberofPointUnder} lists the minimum, average, and maximum numbers of points near the dialog line regarding the additional experiments.
We observed that no point appeared near the dialog line when the number of input classes was set to 10 and 1 in all the additional experiments.
Additionally, the increase and decrease followed the same trend in the additional experiments, shown in Table \ref{NumberofPointUnder}, meaning that the observations obtained in Subsections \ref{MNISTResult} and \ref{CIFAR10Result} are robust on the initial values of network weights.

\begin{figure*}
	\centering
	\includegraphics[width=\linewidth]{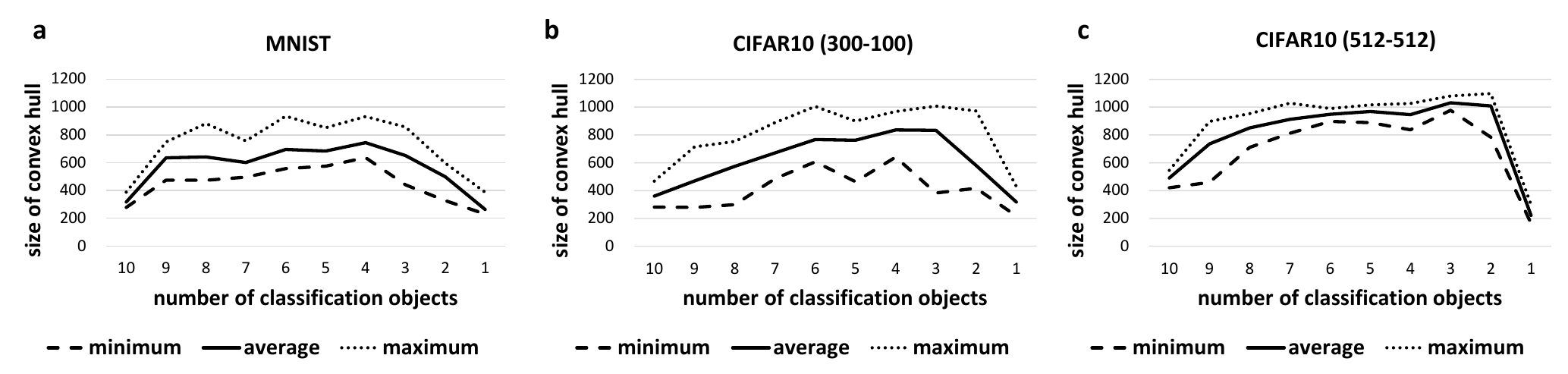}
	\caption{(a)--(c) Size of the convex hull of points in the PH diagrams with MNIST using the FCN $(300,100,10)$, CIFAR-10 using the FCN $(300, 100, 10)$,and CIFAR-10 using FCN $(512, 512, 10)$ by varying the number of input classes, respectively}
	\label{fig:sizeofconvex}
\end{figure*}

\begin{table*}
	\caption{Number of points near the dialog line ($death \le birth + 5$)}
	\label{NumberofPointUnder}
	\centering
	\begin{tabular}{crrrrrrrrr}
		\toprule
			\multirow{2}{*}
			{\begin{tabular}{c}Number of \\
						 input classes
			\end{tabular}
			}
			& \multicolumn{3}{c}{MNIST} 
			& \multicolumn{3}{c}{CIFAR-10 (300-100)} 
			& \multicolumn{3}{c}{CIFAR-10 (512-512)} \\
			& min. & avg. & max.& min. & avg. & max. & min. & avg. &max. \\
		\midrule
		10    &   0 &   0 &   0 &   0 &   0 &   0 &   0 &   0 &   0 \\
		9     &  57 &  96 & 132 &   0 &  11 &  59 &   0 & 115 & 234 \\
		8     & 110 & 150 & 199 &   0 &  33 & 102 &  79 & 273 & 497 \\
		7     & 141 & 209 & 297 &   0 &  78 & 143 & 278 & 375 & 451 \\
		6     & 141 & 269 & 348 &   0 & 136 & 284 & 209 & 376 & 571 \\
		5     & 137 & 332 & 528 &   0 & 142 & 334 &  52 & 380 & 620 \\
		4     & 111 & 308 & 524 &  48 & 196 & 321 &  13 & 423 & 823 \\
		3     &  46 & 131 & 207 &   0 & 158 & 365 & 591 & 764 & 909 \\
		2     &   0 &   0 &   1 &   0 &  36 & 252 & 145 & 581 & 936 \\
		1     &   0 &   0 &   0 &   0 &   0 &   0 &   0 &   0 &   0 \\
		\bottomrule
	\end{tabular}
\end{table*}

\section{Discussion}\label{discusstion}
In this section, the assumptions used in this study are explained and the application of the topological measurement of DNNs is discussed.

\subsection{Assumptions}
The assumptions of this study include the follows: (1) the knowledge in DNNs can be investigated from their network weights among neurons and (2) PH reveals the knowledge complexity of DNNs.
The first assumption is acceptable because the weights are the outcome of the training process.
The second assumption is based on the observations from previous works described in Sec.\ref{sec:intuition} \cite{lecun2015deep,Chollet:2017:DLP:3203489}.
PH reveals the births and deaths of feature combinations, which are difficult to be captured without using PH.
The effectiveness of the second assumption can be evaluated from the usability, which changes depending on the application.

\subsection{Applications}
One of the most important applications of the proposed method is recognizing the quality of DNN training.
In particular, the performance of DNNs can deteriorate for many reasons, including a shortage of data, overfitting, and improper hyper-parameter setting \cite{NIPS2011_4443,srivastava2014dropout}.
Our results imply that the shortage of data can be indicated by the PH, that is the excess of the output neurons produces homologies near the dialog line.
Furthermore, the proposed method is beneficial for selecting appropriate DNN architectures, which is one of the major challenges when utilizing DNNs \cite{saxena2016convolutional,zoph2016neural}. 

\section{Related work}\label{related}
Bianchini et al.\ investigated the upper and lower bounds of network complexity from the viewpoint of PH \cite{bianchini2014complexity}. 
Based on their results, Guss et al.\ empirically analyzed the relationship between the upper bound of network complexity and data complexity measured by PH to determine appropriate network architecture for a given data set \cite{Guss2018OnCT}.
However, these two types of complexities are not homogeneous, and their comparability is uncertain. 
Under these considerations, we addressed the inner representations of DNNs with small perturbations.
Our evaluation results revealed that small perturbations such as the number of output neurons and a variety of input data have significant impact on PH. 
Thus, the sensitivity of PH requires a careful investigation for securing comparability.

Bastian et al.\ investigated the complexity of the inner representation of DNNs using zero-dimensional PH  \cite{rieck2018neural}.
Zero-dimensional PH counts the number of connected components in DNNs.
Fig. \ref{fig:figure2}(f) and (g) have $\beta_{0} = 3$ and $\beta_{0} = 2$ corresponding to the connected components, respectively.
In contrast, the Betti number $\beta_{1}$ reveals the combinations among neurons illustrated in Fig. \ref{fig:figure2}(g), where the neurons one and three collaborate to increase the Betti number $\beta_{1}$. 
Thus, we believe that one-dimensional PH can reveal the combination of neurons and access essential aspects of DNNs that are difficult to be accessed using other methods.

\section{Conclusion}\label{conclusion}
This paper introduced a novel approach to investigate the inner representation of DNNs using PH.
Evaluations were conducted using FCNs and networks combining a CNN and an FCN trained on the MNIST and CIFAR-10 data sets.
The evaluation results demonstrated that the one-dimensional PH of DNNs can reflect both the excess of neurons and problem difficulty, which implies that PH can become one of the prominent methods for investigating the inner representation of DNNs.

The methods for constructing simplicial complexes and defining the filtration are developed on the basis of our attempts. 
The development of these methods will, however, include many research areas, especially due to large variety of network types, including CNNs and recursive neural networks (RNNs).
Furthermore, with regard to computation, the development would require considerable efforts in applying the topological measurement to enlarged neural networks, which can have more than 1,000 layers \cite{he2016deep}.
At the same time, we believe that the topological measurement of DNNs is worth further investigation.


\bibliographystyle{plain}
\bibliography{20210606AnnalsMath-arXiv}
	
%
%

\end{document}